%% file: main.tex
\documentclass{article}

\usepackage{microtype}
\usepackage{graphicx}
\usepackage{subfigure}
\usepackage{booktabs} 
\usepackage[inline]{enumitem}

\usepackage[acronym]{glossaries}
\newacronym{gnn}{GNN}{Graph Neural Network}
\newacronym{ood}{OOD}{out-of-distribution}
\newcommand{\revision}[1]{{\color{black}#1}}

\usepackage{hyperref}

\usepackage[accepted]{icml2023}

\usepackage{amsmath}
\usepackage{amssymb}
\usepackage{mathtools}
\usepackage{amsthm}

\usepackage[capitalize,noabbrev]{cleveref}

\theoremstyle{plain}

\theoremstyle{definition}

\theoremstyle{remark}

\usepackage[textsize=tiny]{todonotes}

\usepackage{color, colortbl}
\usepackage{lipsum}
\usepackage{cleveref}
\usepackage{amssymb}
\usepackage{array}

\newcommand{\xs}[1]{\ensuremath{X^s_{#1}}}
\newcommand{\xc}[1]{\ensuremath{X^c_{#1}}}
\newcommand{\y}[1]{\ensuremath{Y_{#1}}}
\newcommand{\yref}[1]{\ensuremath{Y^R_{#1}}}

\newcommand{\ourmethod}{Hint-ReLIC}

\usepackage{thmtools}
\usepackage{thm-restate}
\DeclareMathOperator*{\expec}{\mathbb E}

\icmltitlerunning{Neural Algorithmic Reasoning with Causal Regularisation}

\begin{document}

\twocolumn[
\icmltitle{Neural Algorithmic Reasoning with Causal Regularisation}



\icmlsetsymbol{equal}{*}
\icmlsetsymbol{equal_adv}{\dag}

\begin{icmlauthorlist}
\icmlauthor{Beatrice Bevilacqua}{yyy}
\icmlauthor{Kyriacos Nikiforou}{equal,comp}
\icmlauthor{Borja Ibarz}{equal,comp}
\icmlauthor{Ioana Bica}{comp}
\icmlauthor{Michela Paganini}{comp}
\icmlauthor{Charles Blundell}{comp}
\icmlauthor{Jovana Mitrovic}{equal_adv,comp}
\icmlauthor{Petar Veli\v{c}kovi\'{c}}{equal_adv,comp}
\end{icmlauthorlist}

\icmlaffiliation{yyy}{Purdue University}
\icmlaffiliation{comp}{DeepMind}

\icmlcorrespondingauthor{Beatrice Bevilacqua}{bbevilac@purdue.edu}

\icmlkeywords{Machine Learning, ICML}

\vskip 0.3in
]



\printAffiliationsAndNotice{\icmlEqualContribution \icmlEqualAdvising} 

\begin{abstract}
Recent work on neural algorithmic reasoning has investigated the reasoning capabilities of neural networks, effectively demonstrating they can learn to execute classical algorithms on unseen data coming from the train distribution.
However, the performance of existing neural reasoners significantly degrades on out-of-distribution (OOD) test data, where inputs have larger sizes.
In this work, we make an important observation: there are many \emph{different} inputs for which an algorithm will perform certain intermediate computations \emph{identically}. This insight allows us to develop data augmentation procedures that, given an algorithm's intermediate trajectory, produce inputs for which the target algorithm would have \emph{exactly} the same next trajectory step.
\revision{We ensure invariance in the next-step prediction across such inputs, by employing a self-supervised objective derived by our observation, formalised in a causal graph.
We prove that the resulting method, which we call \ourmethod,} improves the OOD generalisation capabilities of the reasoner. We evaluate our method on the CLRS algorithmic reasoning benchmark, where we show up to 3$\times$ improvements on the OOD test data. 
%
\end{abstract}

\section{Introduction}

Recent works advocate for building neural networks that can reason \citep{xu2019can,xu2021how,velivckovic2021neural,velivckovic2022clrs}. 
Therein, it is posited that combining the robustness of algorithms with the flexibility of neural networks can help us accelerate progress towards models that can tackle a wide range of tasks with real world impact \citep{davies2021advancing,deac2021neural,velickovic2022reasoning,bansal2022end,beurer2022learning}.
The rationale is that, if a model learns how to reason, or learns to execute an algorithm, it should be able to apply that reasoning, or algorithm, to a completely novel problem, even in a different domain.
Specifically, if a model has learnt an algorithm, it should be gracefully applicable on out-of-distribution (OOD) examples, which are substantially different from the examples in the training set, and return correct outputs for them. 
This is because an algorithm---and reasoning in general---is a sequential, step-by-step process, where a simple decision is made in each step based on outputs of the previous computation. 

Prior work \citep{diao2022relational,dudzik2022graph,ibarz2022a,mahdavi2022towards} has explored this setup, using the CLRS-30 benchmark~\citep{velivckovic2022clrs}, and showed that while many algorithmic tasks can be learned by Graph Neural Network (GNN) processors in a way that generalises to larger problem instances, there are still several algorithms where this could not be achieved.

Importantly, CLRS-30 also provides ground-truth \emph{hints} for every algorithm. Hints correspond to the state of different variables employed to solve the algorithm (e.g. positions, pointers, colouring of nodes) along its trace. Such hints can optionally be used during training, but are not available during evaluation.
In previous work, they have mainly been used as auxiliary targets together with the  algorithm output. The prevailing hypothesis is that gradients coming from predicting these additional relevant signals will help constrain the representations in the neural algorithmic executor and prevent overfitting. Predicted hints can also be optionally fed back into the model to provide additional context and aid their prediction at the next step. 

\begin{figure*}[ht!!]
    \includegraphics[width=\linewidth]{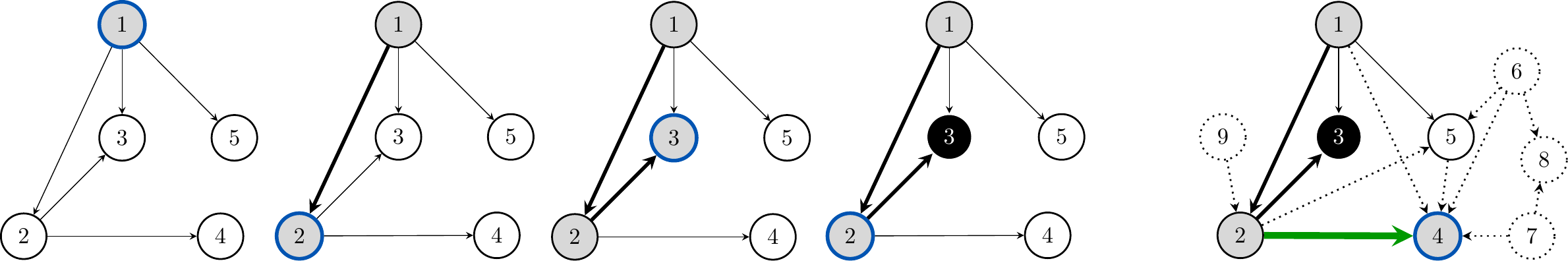}
    \caption{An illustration of the key observation of our work, on the depth-first search (DFS) algorithm as implemented in CLRS-30 \citep{velivckovic2022clrs}. On the left, the first four steps of DFS are visualised. At each step, DFS explores the unvisited neighbour with the smallest index, and backtracks if no unexplored neighbours exist. The next computational step---\emph{assigning $2$ as the parent of $4$}---is bound to happen, even under many \emph{transformations} of this graph. For example, if we were to insert new (dashed) nodes and edges into the graph, this step would still proceed as expected. Capturing this computational invariance property is the essence of our paper.}\label{fig:key}
\end{figure*}

In practice, while utilising hints in this way does lead to models that \emph{follow} the algorithmic trajectory better, they have had a less substantial impact on the accuracy of the predicted \emph{final output}. This is likely due to the advent of powerful strategies such as recall \citep{bansal2022end}, wherein the input is fed back to the model at every intermediate step, constantly ``reminding'' the model of the problem that needs to be solved. The positive effect of recall on the final output accuracy has been observed on many occasions \citep{mahdavi2022towards}, and outweighs the contribution from directly predicting hints and feeding them back.

In this work, we propose a method, namely \ourmethod, that decisively demonstrates an advantage to using hints.
We base our work on the observation that there are many different inputs for which an algorithm will make \emph{identical} computations at a certain step (Figure \ref{fig:key}). For example, applying the bubble sort algorithm from the left on $[2,1,3]$ or $[2,1,\underbar{5},3]$ will result in the same first step computation: a comparison of $2$ and $1$, followed by swapping them. Conversely, the first step of execution would be different for inputs $[2,1,3]$ and $[2,\underbar{5},1,3]$; the latter input would trigger a comparison of $2$ and $5$ without swapping them.
This observation allows us to move beyond the conventional way of using hints, i.e. autoregressively predicting them~\citep{velivckovic2022clrs}. Instead, we design a novel way that learns more informative representations that enable the networks to more faithfully execute algorithms. Specifically, we learn representations that are similar for inputs that result in identical intermediate computation.
First, we design a causal graph in order to formally model an algorithmic execution trajectory.
Based on this, we derive a self-supervised objective for learning hint representations that are invariant across inputs having the same computational step. 
Moreover, we prove that this procedure will result in stronger causally-invariant representations.
\paragraph{Contributions.}
Our three key contributions are as follows:
\begin{enumerate}
    \item We design a \emph{causal graph} capturing the observation that the execution of an algorithm at a certain step is determined only by a subset of the input;
    \item Motivated by our causal graph, we present a \emph{self-supervised objective} to learn representations that are \emph{provably} invariant to changes in the input subset that does not affect the computational step;
    \item We test our model, dubbed \ourmethod, on the CLRS-30 algorithmic reasoning benchmark~\citep{velivckovic2022clrs}, demonstrating a \emph{significant improvement} in out-of-distribution generalisation over the recently published state-of-the-art \citep{ibarz2022a}.
\end{enumerate}

\section{Related Work}
\paragraph{GNNs and invariance to size shifts.}
Graph Neural Networks (GNNs) constitute a popular class of methods for learning representations of graph data, and they have been successfully applied to solve a variety of problems.
We refer the reader to \citet{bronstein2021geometric, jegelka2022theory} for a thorough understanding of GNN concepts.
While GNNs are designed to work on graphs of any size, recent work has empirically shown poor size-generalisation capabilities of standard methods, mainly in the context of molecular modeling~\citep{Gasteiger2022HowDG}, graph property prediction~\citep{corso2020principal}, and in executing specific graph algorithms~\citep{velivckovic2019neural,joshi2020LearningTT}. A theoretical study of failure cases has been recently provided in \citet{xu2021how}, with a focus on a geometrical interpretation of OOD generalisation. In order to learn models performing equally well in- and out-of-distribution, \citet{bevilacqua2021size,chen2022learning,zhou2022ood} designed ad-hoc solutions satisfying assumed causal assumptions. However, these models are not applicable to our setting, as the assumptions on our data generation process are significantly different. With the same motivation, \citet{buffellisizeshiftreg} introduced a regularisation strategy to improve generalisation to larger sizes, while \citet{yehudai2021local} proposed a semi-supervised and a self-supervised objective that assume access to the test distribution. However, these models are not designed to work on algorithmic data, where OOD generalisation is still underexplored.

\paragraph{Neural Algorithmic Reasoning.}
In order to learn to execute algorithmic tasks, a neural network must include a \emph{recurrent} component simulating the individual algorithmic steps. This component is applied a variable number of times, as required by the size of the input and the problem at hand. The recurrent component can be an LSTM~\citep{gers2001LSTMRN}, possibly augmented with a memory as in Neural Turing Machines~\citep{graves2014NeuralTM, graves2016hybrid}; it could exploit spatial invariances in the algorithmic task through a convolutional architecture~\citep{bansal2022end}; it could be based on the transformer self-attentional architecture, as in the Universal Transformer~\citep{dehghani2018UT}; or it could be a Graph Neural Network (GNN).
%
GNNs are particularly well suited for algorithmic execution \citep{velivckovic2019neural,xu2019can}, and they have been applied to algorithmic problems before with a focus on extrapolation capabilities~\citep{palm2017RecurrentRN,selsam2018LearningAS,joshi2020LearningTT, tang2020TowardsSG}. Recently,~\citet{velivckovic2021neural} have proposed a general framework for algorithmic learning with GNNs. To reconcile different data encodings and provide a unified evaluation procedure, ~\citet{velivckovic2022clrs} have presented a benchmark of algorithmic tasks covering a variety of areas.
This benchmark, namely the CLRS algorithmic benchmark, represents data as graphs, showing that the graph formulation is general enough to include several algorithms, and not just the graph-based ones. On the CLRS benchmark, \citet{ibarz2022a} has recently presented several improvements in the architecture and learning procedure in order to obtain better performances. However, even the latest state-of-the-art models suffer from performance drops in certain algorithms when going out-of-distribution, an aspect we wish to improve upon here.

\paragraph{Self-supervised learning.} 
Recently, many self-supervised representation learning methods that achieve good performance on a wide range of downstream vision tasks without access to labels have been proposed.
One of the most popular approaches relies on contrastive objectives that make use of data augmentations to solve the instance discrimination task \citep{wu2018unsupervised, chen2020simple, he2020momentum, mitrovic2021representation}. Other approaches that rely on target networks and clustering have also been explored~\citep{grill2020bootstrap,caron2020unsupervised}. Our work is similar in spirit to \citet{mitrovic2021representation}, which examines representation learning through the lens of causality and employs techniques from invariant prediction to make better use of data augmentations.
This approach has been demonstrated to be extremely successful on vision tasks \citep{tomasev2022pushing}.
In the context of graphs,  \citet{You2020GraphCL,suresh2021adversarial,you2022bringing} have studied how to learn contrastive representations, with particular attention paid to data augmentations. Moreover, \citet{velickovic2018deep,zhu2020deep} proposed novel  objectives based on mutual information maximization in the graph domain to learn representations.
Several other self-supervised methods (e.g.\ \citet{thakoor2021bootstrapped}) have also been studied, and we refer the reader to \citet{xie2022self} for a review of existing literature \revision{on self-supervision with GNNs}.

\section{Causal Model for Algorithmic Trajectories}
\begin{figure}[t]
\vskip 0.2in
\begin{center}
\centerline{\includegraphics[width=\columnwidth]{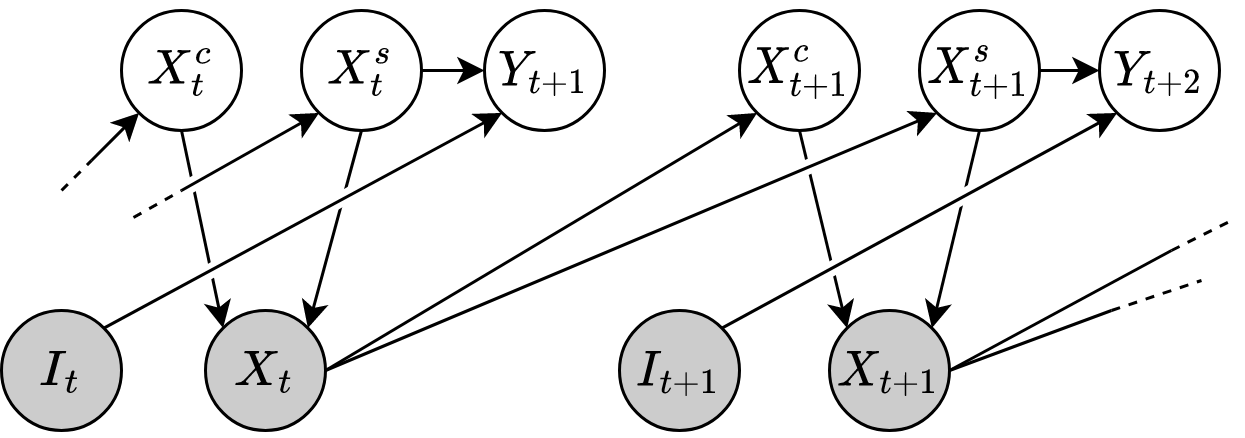}}
\caption{The causal graph formalising our assumption about the outcome of a step depends only on a subset $\xs{t}$ of the snapshot $X_t$, while the remainder $\xc{t}$ of the snapshot can be arbitrarily different.}
\label{fig:data-gen}
\end{center}
\vskip -0.2in
\end{figure}
An algorithm's execution trajectory is described in terms of the \emph{inputs}, \emph{outputs} and \emph{hints}, which represent intermediate steps in the execution. We consider a graph-oriented way of representing this data~\citep{velivckovic2022clrs}: inputs and outputs are presented as data on nodes and edges of a graph, and hints are encoded as node, edge or graph features changing over time steps.

To better understand the data at hand, we propose to formalise the data generation process for an algorithmic trajectory using a \emph{causal graph}. In such a causal graph, nodes represent random variables, and incoming arrows indicate that the node is a function of its parents \citep{pearl2009causality}.
The causal graph we use can be found in \Cref{fig:data-gen}. Note that this graph does not represent input data for the model, but a way of describing how any such data is generated.

Let us consider the execution trajectory of a certain algorithm of interest, at a particular time step $t$. Assume $X_1$ to be the observed input, and let $X_t$ be the random variable denoting the ``snapshot'' at step $t$ of the algorithm execution on the input. For example, in bubble sort, $X_1$ will be the initial (unsorted) array, and $X_t$ the array after $t$ steps of the sorting procedure (thus a partially-sorted array).

The \emph{key contribution} of our causal graph is modelling the assumption that outcomes of a particular execution step depend only on a subset of the current snapshot, while the remainder of the snapshot can be arbitrarily different. 

Accordingly, we assume the snapshot $X_t$ to be generated from \emph{two} random variables, $\xc{t}$ and $\xs{t}$, with $\xc{t}$ representing the part of the snapshot that does not influence the current execution step (what can be changed without affecting the execution), while $\xs{t}$ the one that determines it (what needs to be stable). 

\begin{figure*}[t]
\vskip 0.2in
\begin{center}
\centerline{\includegraphics[width=\textwidth]{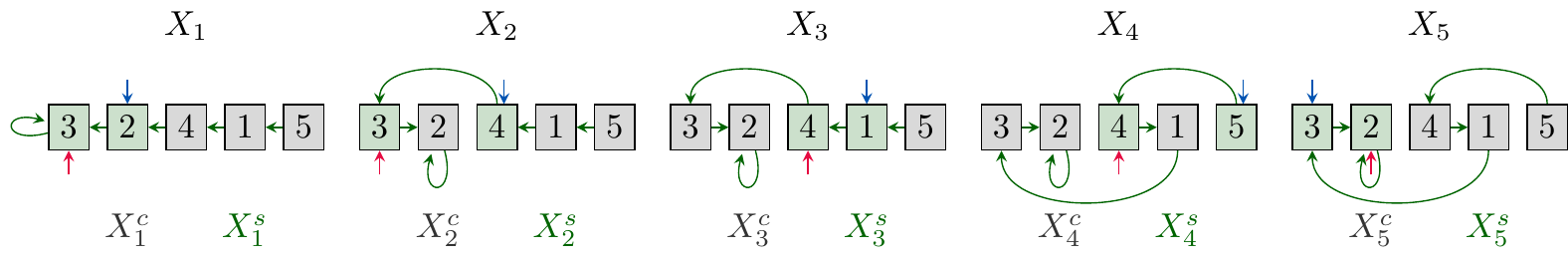}}
\caption{Example of values of $\xc{t}$ and $\xs{t}$ on an input array in the execution of the bubble sort algorithm. At every step of computation, bubble sort compares and possibly swaps exactly \emph{two} nodes---those nodes are the only ones determining the outcome of the current step, and hence they constitute $\xs{t}$. All other nodes are part of $\xc{t}$.}
\label{fig:bubblesort}
\end{center}
\vskip -0.2in
\end{figure*}

Let us now revisit our bubble sort example from this perspective (see \Cref{fig:bubblesort}). At each execution step, bubble sort compares two adjacent elements of the input list, and swaps them if they are not correctly ordered. Hence, in this particular example, $\xs{t}$ constitutes these two elements being compared at step $t$, while the remaining elements---which do not affect whether or not a swap is going to happen at time $t$---form $\xc{t}$. By definition this implies that the next algorithm state is a function of \emph{only} $\xs{t}$.

The data encoding used by~\citet{velivckovic2022clrs} prescribes that hints have values provided in \emph{all} relevant parts of the graph. That is, in a graph of $n$ nodes, an $m$-dimensional node hint has shape $\mathbb{R}^{n\times m}$, and an $m$-dimensional edge hint has shape $\mathbb{R}^{n\times n\times m}$. However, in order to keep our causal model simple, we choose to track the next-step hint in \emph{only one} of those values, using an \emph{index}, $I_t$, to decide which. Specifically, $I_t\in\{1, 2, \dots, n\}$ are possible indices for node-level hints, and $I_t\in\{(1, 1), (1, 2), \dots, (1, n), (2, 1), \dots, (2, n), \dots, (n, n)\}$ are possible indices for edge-level hints. For the indexed node/edge only, our causal graph then tracks the next-step value of the hint (either no change from the previous step or the new value), which we denote by $\y{t+1}$.

Returning once again to our bubble sort example: one specific hint being tracked by the algorithm is which two nodes in the input list are currently considered for a swap. If $I_2=4$, then $\y{3}$ will track whether node $4$ is being considered for a swap, immediately after two steps of the bubble sort algorithm have been executed.

Once step $t$ of the algorithm has been executed, a new snapshot $X_{t+1}$ is produced, and it can be decomposed into $\xc{t+1}$ and $\xs{t+1}$, just as before. Note that the execution in CLRS-30 is assumed \emph{Markovian} \citep{velivckovic2022clrs}: the snapshot at step $t$ contains all the information to determine the snapshot at the next step. Finally, the execution terminates after $T$ steps, and the final output is produced. We can then represent the output in a particular node/edge---indexed by $I_T$, just as before---by $\y{T+1}^o := g(\xs{T}, I_T)$, with $g$ being the function producing the algorithm output.

As can be seen in \Cref{fig:data-gen}, $\xs{t}$ has all the necessary information to predict $\y{t+1}$, since our causal model encodes the conditional independence assumption $\y{t+1} \perp \xc{t} \, \vert \, \xs{t} $. More importantly, using the independence of mechanisms~\citep{peters2017elements} we can conclude that under this causal model, performing interventions on $\xc{t}$ by changing its value, does not change the conditional distribution $P(\y{t+1} \,  \vert \, \xs{t})$.
Note that this is exactly the formalisation of our initial intuition: \emph{the output of a particular step of the algorithm (i.e.,\ $\y{t+1}$) depends only on a subset of the current snapshot (i.e., $\xs{t}$), and thus it is not affected by the addition of input items that do not interfere with it} (which we formalise as an \textbf{intervention} on $\xc{t}$).\footnote{In bubble sort, adding sorted keys at the end of the array does not affect whether we are swapping the current entries.} Therefore, given a step $t \in [1 \dots T]$, for all $x, x' \in \mathcal{X}^c_t$, where $\mathcal{X}^c_t$ denotes the domain of $\xc{t}$, we have that $\xs{t}$ is an \emph{invariant} predictor of $\y{t+1}$ under interventions on $\xc{t}$:
\begin{align}
\label{eq:invariantx}
    p^{\text{do}(\xc{t})=x}(\y{t+1} \vert \xs{t}) = p^{\text{do}(\xc{t})=x'}(\y{t+1} \vert \xs{t}),
\end{align}
where $p^{\text{do}(\xc{t})=x}$ denotes the distribution obtained from assigning $\xc{t}$ the value of $x$, i.e.\ the interventional distribution.

Note, however, that \Cref{eq:invariantx} does not give us a practical way of ensuring that our neural algorithmic reasoner respects these causal invariances, because it only has access to the entirety of the current snapshot $X_t$, without knowing its specific subsets $\xc{t}$ and $\xs{t}$. More precisely, it is generally not known \emph{which} input elements constitute $\xs{t}$. \revision{For this reason, $\xc{t}$ and $\xs{t}$ are represented as unobserved random variables (white nodes) in \Cref{fig:data-gen}.} In the next section, we will describe how to ensure invariant predictions for our reasoner, leveraging only $X_t$.

\section{Size-Invariance through Self-Supervision in Neural Algorithmic Reasoning}\label{sec:self-sup}
Given a step $t$, to ensure invariant predictions of $\y{t+1}$ without access to $\xs{t}$, we construct a \emph{refinement} task $\yref{t+1}$ and learn a representation $f(X_t, I_t)$ to predict $\yref{t+1}$, as originally proposed for images in~\citet{mitrovic2021representation}. A refinement for a task \citep{chalupka2014visual} represents a more fine-grained version of the initial task.

More formally, given two tasks $R : \mathcal{A} \rightarrow \mathcal{B}$ and $T : \mathcal{A} \rightarrow \mathcal{B}^\prime$,
task $R$ is more (or equally) fine-grained than task $T$ if, for any two elements $a, a^\prime \in \mathcal{A}$, $R(a) = R(a^\prime) \implies T(a) = T(a^\prime)$.

We will use this concept to show that a representation learned on the refinement task can be effectively used in the original task. 
Note that, as for $\y{t+1}$, we assume $f(X_t, I_t)$ to be the representation learned from $X_t$ of a predefined hint value---indexed by $I_t$---for example, the representation of the predecessor of a specific element of the input list.

Given a step $t$, let $\yref{t+1}$ be a refinement of $\y{t+1}$, and let $f(X_t, I_t)$ be a representation learned from $X_t$, used for the prediction of the refinement (see \Cref{fig:causal}).   
As we will formally prove, a representation that is invariant in the prediction of the refinement task across changes in $\xc{t}$ is also invariant in the prediction of the algorithmic step under these changes. Therefore, optimising $f(X_t, I_t)$ to be an invariant predictor for the refinement task $\yref{t+1}$ represents a \emph{sufficient} condition for the invariance in the prediction of the next algorithmic state, $\y{t+1}$.

In the next subsection we present how to learn $f(X_t, I_t)$ in order to be an invariant predictor of $\yref{t+1}$ under changes in $\xc{t}$. Then, we show that this represents a sufficient condition for $f(X_t, I_t)$ to be an invariant predictor of $\y{t+1}$ across changes in $\xc{t}$.

\subsection{Learning an invariant predictor of the refinement}
\begin{figure}[t]
\vskip 0.2in
\begin{center}
\centerline{\includegraphics[width=\columnwidth]{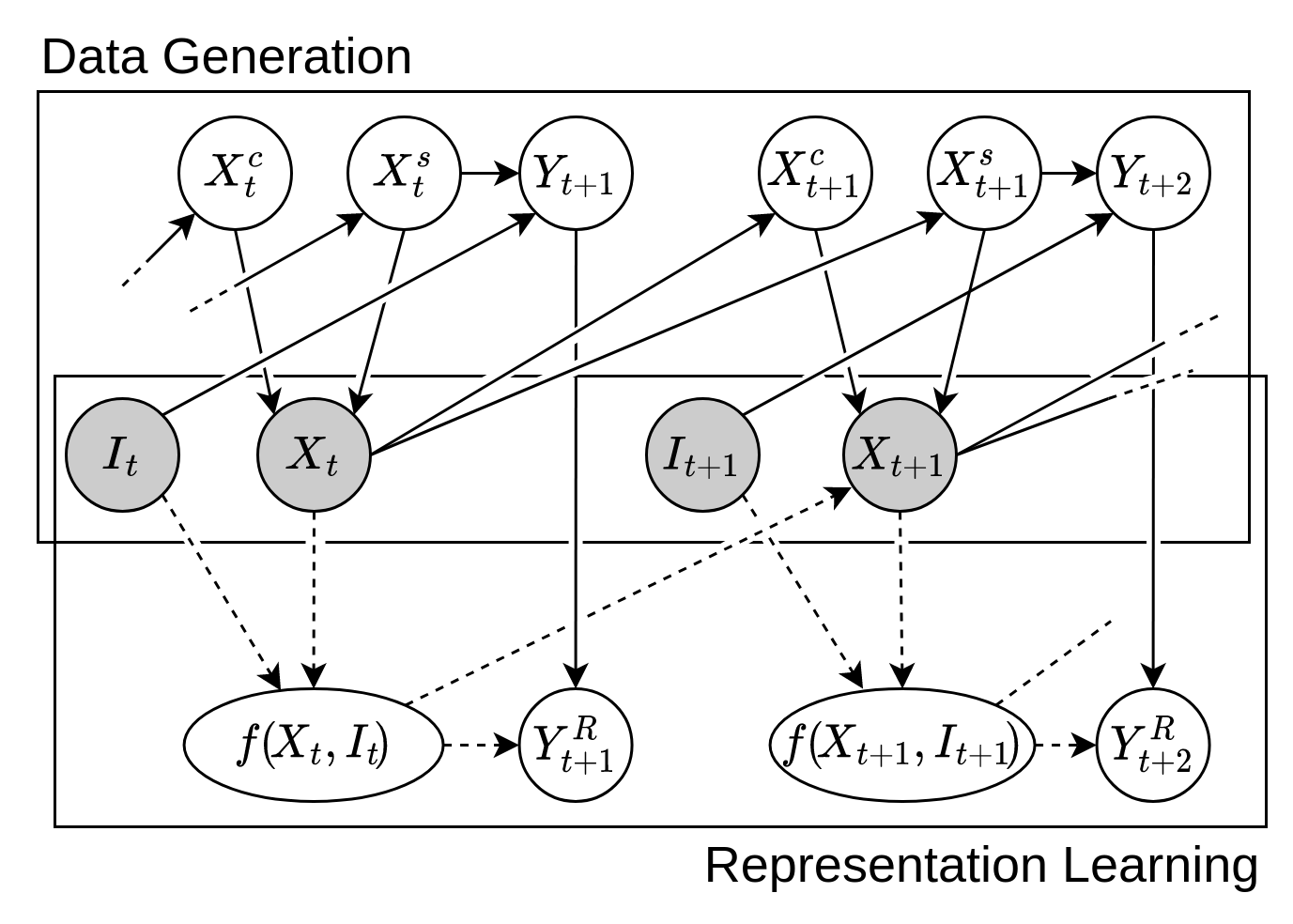}}
\caption{Our causal graph with the inclusion of the representation learning components as in \citet{mitrovic2021representation}. \revision{Solid arrows represent the causal relationships. Dashed arrows represent what is used to learn (in the case of $f(X_t, I_t)$) or predict (in the case of $\yref{t+1}$) the corresponding random variables.}}
\label{fig:causal}
\end{center}
\vskip -0.2in
\end{figure}
We consider $\yref{t+1}$ to be the most-fine-grained refinement task, which corresponds to classifying each (hint) instance individually, that is, a contrastive learning objective where we want to distinguish each hint from all others. This represents the \emph{most-fine-grained refinement}, because $\yref{t+1}(a)=\yref{t+1}(a')\Longleftrightarrow a=a'$, by definition.  
Our goal is to learn $f(X_t, I_t)$ to be an invariant predictor of $\yref{t+1}$ under changes (interventions) of $\xc{t}$. Thus, given a step $t \in [1 \dots T]$, for all $x, x' \in \mathcal{X}^c_t$, we want $f(X_t, I_t)$ such that
\begin{align}
\label{eq:invariantfx}
    p^{\text{do}(\xc{t})=x}(\yref{t+1} \vert f(X_t, I_t)) = p^{\text{do}(\xc{t})=x'}(\yref{t+1} \vert f(X_t, I_t)),
\end{align}
where $p^{\text{do}(\xc{t})}$ is the interventional distribution and $\mathcal{X}^c_t$ denotes the domain of $\xc{t}$. Since we do not have access to $\xc{t}$, as it is unobserved (\revision{it is a white node in} \Cref{fig:data-gen,fig:causal}), we cannot explicitly intervene on it. Thus, we simulate interventions on $\xc{t}$ through data augmentation.

As we are interested in being invariant to appropriate size changes, we design a data augmentation procedure tailored for neural algorithmic reasoning, which mimics interventions changing the size of the input. Given a current snapshot of the algorithm on a given input, the data augmentation procedure should produce an augmented input which is larger, but on which the execution of the current step is going to proceed identically.

For example, a valid augmentation in bubble sort at a certain step consists of adding new elements to the tail of the input list, since the currently-considered swap will occur (or not) regardless of any elements added there. Thus, the valid augmentations for the bubble sort algorithm at a given step are all possible ways to add items in such a way that ensures that the one-step execution is unaffected by this addition.

To learn an encoder $f(X_t, I_t)$ that satisfies \Cref{eq:invariantfx}, we propose to explicitly enforce invariance under valid augmentations. Such augmentations, as discussed, provide us with diverse inputs with an identical intermediate execution step. 

Specifically, we use the ReLIC objective~\citep{mitrovic2021representation} as a regularisation term, which we adapt to our causal graph as follows. Consider a time step, $t$, and let $\mathcal{D}_t$ be the dataset containing the snapshots at time $t$ for all the inputs. Let $i_t, j_t \in I_t$ be two indices, and denote by $a_{lk} = (a_l, a_k) \in \mathcal{A}_{x_t} \times \mathcal{A}_{x_t}$ a pair of augmentations, with $\mathcal{A}_{x_t}$ the set of all possible valid augmentations at $t$ for $x_t$ (which simulate the interventions on $\xc{t}$). 

The objective function to optimise becomes:
\begin{align} \label{eq:contrastive}
    & \mathcal{L}_t = \nonumber \\
    & \!\! - \!\! \sum_{x_t \in \mathcal{D}_t} \!\! \Big( \sum_{i_t} \sum_{a_{lk}} \log \frac{\exp{(\phi(f(x^{a_l}_t, i_t), f(x^{a_k}_t, i_t)))}}{\sum_{j_t\neq i_t}\exp{(\phi(f(x^{a_l}_t, i_t), f(x^{a_k}_t, j_t)))}} \nonumber \\
    &\qquad - \alpha \sum_{a_{lk}, a_{qm}} \text{KL}(p^{\text{do}(a_{lk})}, p^{\text{do}(a_{qm})})\Big)
\end{align}
with $x^{a}_t$ the data augmented with augmentation $a$, and $\alpha$ a weighting of the KL divergence penalty. 
The first term represents a contrastive objective where we compare a hint representation in $x^{a_l}_t$, namely $f(x^{a_l}_t, i_t)$, with all the possible representations in $x^{a_k}_t$, $f(x^{a_k}_t, j_t)$. Note that this is different from standard contrastive objectives, where negative examples are taken from the batch. Due to space constraints, we expand on the derivation of \Cref{eq:contrastive} in \Cref{app:contrastive}.

In practice, we consider only one augmentation per graph, which is equivalent to setting $a_l$ to the identity transformation.
\revision{
Consequently, the hint representation in the original graph $f(x^{a_l}_t, i_t)$ is regularised to be similar to the hint representation in the augmentation $f(x^{a_k}_t, i_t)$ and dissimilar to all other possible representations in the augmentation $f(x^{a_k}_t, j_t)$, $j_t \neq i_t$. Similarly, the hint representation in the augmentation $f(x^{a_k}_t, i_t)$ is regularised to be similar to the hint representation in the original graph $f(x^{a_l}_t, i_t)$ and dissimilar to all other possible representations in the original graph $f(x^{a_l}_t, j_t)$, $j_t \neq i_t$.
}

We follow the standard setup in contrastive learning and implement $\phi(f(x^{a_l}_t, i_t), f(x^{a_k}_t, i_t)) = \langle\,h(f(x^{a_l}_t), i_t), h(f(x^{a_k}_t, i_t))\,\rangle / \tau$
with $h$ a fully-connected neural network and $\tau$ a temperature parameter.
Finally, we use a KL penalty to ensure invariance in the probability distribution across augmentations. This is a requirement for satisfying the assumptions of our key theoretical result.
\begin{figure}[t]
\vskip 0.2in
\begin{center}
\centerline{\includegraphics[width=\columnwidth]{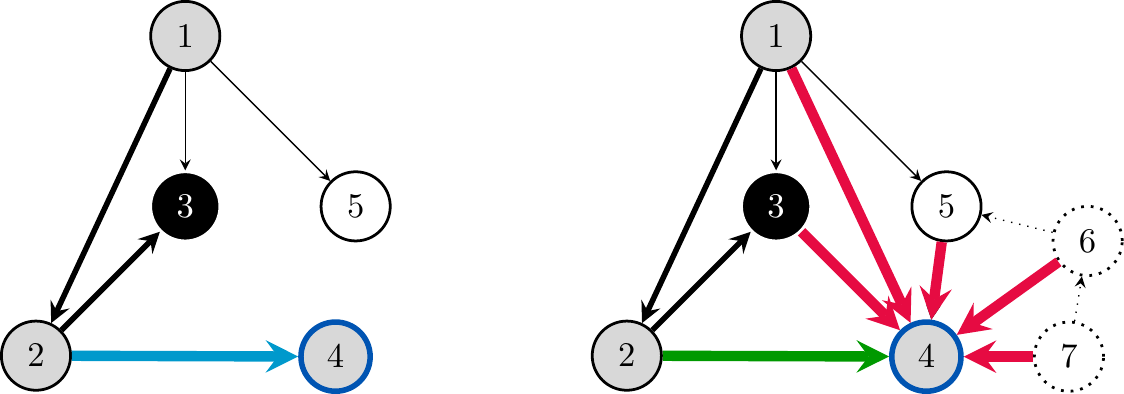}}
\caption{Example of applying our data augmentation and contrastive loss, following the example in Figure \ref{fig:key}. An input graph (left) is augmented by adding nodes and edges (right), such that the next step---making $2$ the parent of $4$, i.e. $\pi_4 = 2$---remains the same. The representation of the pair $(4, 2)$ is hence contrasted against all other representations of pairs $(4, u)$ in the augmented graph. In other words, the green edge is the \emph{positive} pair to the blue edge, with other edges (in red) being \emph{negative} pairs to it.}
\label{fig:scc}
\end{center}
\vskip -0.2in
\end{figure}

\paragraph{Example.}
To better understand \Cref{eq:contrastive}, we provide an example illustrated in \Cref{fig:scc}. We will consider one of the algorithms in CLRS-30---Kosaraju's strongly connected component (SCC) algorithm \citep{aho1974design}---which consists of two invocations of depth-first search (DFS). 

Let $G=(V,E)$ be an input graph to the SCC algorithm. Further, assume that at step $t$, the algorithm is visiting a node $v\in V$. We will focus on the prediction of the \emph{parent} of $v$: the node from which we have reached $v$ in the current DFS invocation. Note that, in practice, this is a classification task where node $v$ decides which of the other nodes is its parent. Accordingly, given a particular node $v$, our model computes a representation for every other node $u\in V$. This representation is which is then passed through a final classifier, outputting the (unnormalised) probability of $u$ being the parent of $v$.  

Now, consider any augmentation of $G$'s nodes and edges that does not disrupt the current step of the search algorithm, denoted by $G^a = (V^a, E^a)$. For example, as the DFS implementation in CLRS-30 prefers nodes with a smaller \texttt{id} value, a valid augmentation can be obtained by adding nodes with a larger \texttt{id} than $v$ to $V^a$, and adding edges from them to $v$ in $E^a$ \revision{(dashed nodes and edges in \Cref{fig:scc})}. Note that this augmentation does not change the predicted parent of $v$. We can enforce that our representations respect this constraint by using our regularisation loss in \cref{eq:contrastive}.

Given a node $v\in V$, we denote the representation of its parent node, $\pi_v\in V$ by $f(G, (v, \pi_v))$. This representation is contrasted to \emph{all} other representations of nodes $w\in V^a$ in the augmented graph, that is $f(G^a, (v, w))$.\footnote{Note that, in this case, $I_t$ is a two-dimensional index, choosing two nodes---i.e., an edge---at once.}

More precisely, the most similar representation of $f(G, (v, \pi_v))$ is the representation \emph{in the augmentation} of the parent of $v$, $f(G^a, (v, \pi_v))$, while the representations associated to all other nodes (including the added ones) represent the negative examples $f(G^a, (v, w))$, for $w \neq \pi_v$.

\revision{\Cref{fig:scc} illustrates the prediction of the parent of node $v=4$. In this case, $i_t$ in \cref{eq:contrastive} indexes the true parent of node $4$, namely $\pi_4 = 2$, and therefore $i_t = (4,2)$, while $j_t$ iterates over all other possible indices of the form $(4,u)$, $u \in V^a$, indeed representing \emph{all other possible parents} of $4$. The objective of \cref{eq:contrastive} is to make the true parent representation in the original graph $f(G, (4,2))$ similar to the true parent representation in the augmentation $f(G^{a}, (4,2))$, and dissimilar to the representations of the other possible parents in the augmentation $f(G^{a}, (4,u))$, $u \in V^a$. The same process applies to the augmentation.
}

\subsection{Implications of the invariance}
In the previous subsection, we have presented a self-supervised objective, justified by our assumed causal graph, in order to learn invariant predictors for a refinement task $\yref{t+1}$ under changes of $\xc{t}$. However, our initial goal was to ensure invariance in the prediction of algorithmic hints $\y{t+1}$ across $\xc{t}$. Now we will bridge these two aims.

In the following, we show how learning a representation that is an invariant predictor of $\yref{t+1}$ under changes of $\xc{t}$ represents a \emph{sufficient} condition for this representation to be invariant to $\xc{t}$ when predicting $\y{t+1}$.

\begin{restatable}{theorem}{invariancetheo}
\label{theo:inv}
Consider an algorithm and let $t \in [1 \dots T]$ be one of its steps.
Let $\y{t+1}$ be the task representing a prediction of the algorithm step and let $\yref{t+1}$ be a refinement of such task. If $f(X_t, I_t)$ is an invariant representation for $\yref{t+1}$ under changes in $\xc{t}$, then $f(X_t, I_t)$ is an invariant representation for $\y{t+1}$ under changes in $\xc{t}$, that is, for all $x, x' \in \mathcal{X}^c_t$, the following holds:
\begin{align*}
    &p^{\text{do}(\xc{t})=x}(\yref{t+1} \vert f(X_t, I_t)) = p^{\text{do}(\xc{t})=x'}(\yref{t+1} \vert f(X_t, I_t)) \\
    &\implies
    \\
    &p^{\text{do}(\xc{t})=x}(\y{t+1} \vert f(X_t, I_t)) = p^{\text{do}(\xc{t})=x'}(\y{t+1} \vert f(X_t, I_t)).
\end{align*}
\end{restatable}
We prove \Cref{theo:inv} in \Cref{app:theory}. Note that this justifies our self-supervised objective: by learning invariant representations though a refinement task, we can also guarantee invariance in the hint prediction. In other words, we can \emph{provably} ensure that the prediction of an algorithm step is not affected by changes in the input that do not interfere with the current execution step. Since we can express these changes in the form of \emph{addition} of input nodes, we are ensuring that the hint prediction is the same on two inputs of different sizes, but identical current algorithmic step.

\section{Experiments}
We conducted an extensive set of experiments to answer the following main questions: \begin{enumerate}
    \item \emph{Can our model, \ourmethod, which relies on the addition of our causality-inspired self-supervised objective, outperform the corresponding base model in practice?}
    \item \emph{What is the importance of such objective when compared to other changes made with respect to the previous state-of-the-art model?}
    \item \emph{How does \ourmethod{} compare to a model which does not leverage hints at all, directly predicting the output from the input? Are hints necessary?}
\end{enumerate}

\paragraph{Model.} As a base model, we use the Triplet-GMPNN architecture proposed by \citet{ibarz2022a}, which consists of a fully-connected MPNN~\cite{gilmer2017neural} where the input graph is encoded in the edge features, augmented with gating and triplet reasoning \citep{dudzik2022graph}.

We replace the loss for predicting the next-step hint in the base model with our regularisation objective (\Cref{eq:contrastive}), which aims at learning hint representations that are \emph{invariant to size changes that are irrelevant to the current step} via constrastive and KL losses.

We make an additional change with respect to the base model, consisting of including the \emph{reversal} of hints of \texttt{pointer} type. More specifically, given an input graph, if a node $A$ points to another node $B$ in the graph, we include an additional (edge-based) hint representing the pointer from $B$ to $A$. This change (which we refer to as \textbf{reversal} in the results) consists simply in the inclusion of these additional hints, and we study the impact of this addition in \Cref{exp:ablation}. The resulting model is what we call \ourmethod.

\begin{figure*}[t]
\vskip 0.2in
\begin{center}
\centerline{\includegraphics[width=\textwidth]{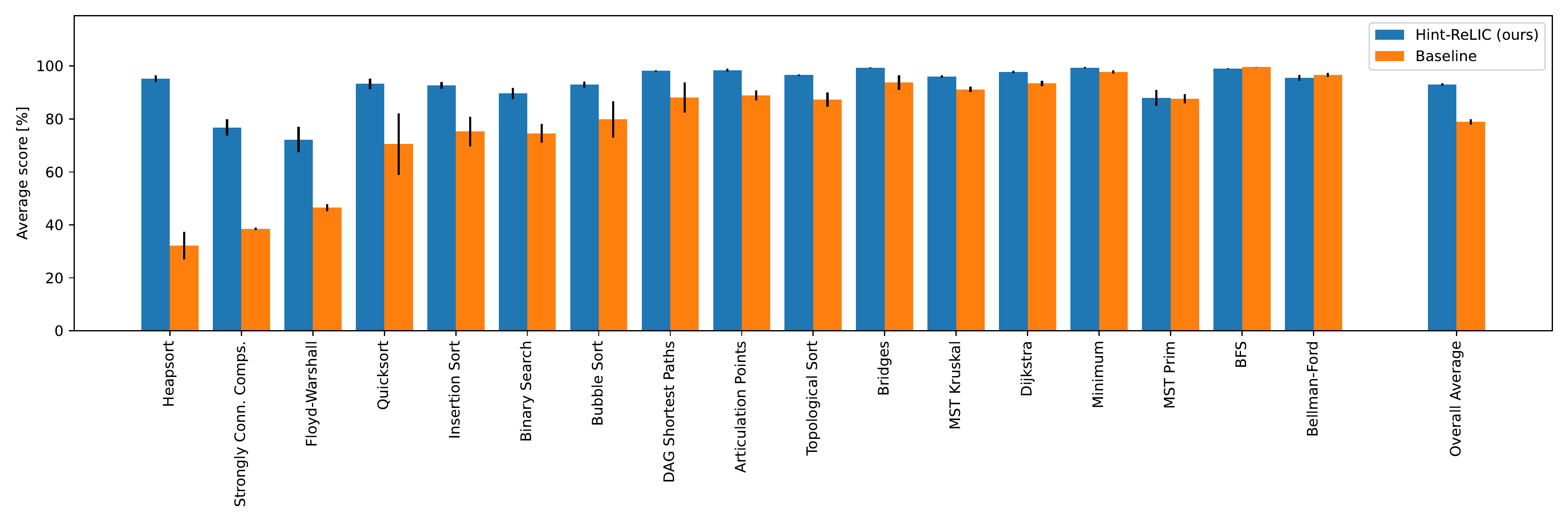}}
\caption{Per-algorithm comparison of the Triplet-GMPNN baseline \citep{ibarz2022a} and our \ourmethod. Error bars represent
the standard error of the mean across three random seeds. The final column shows the average and standard error of the mean performances across the different algorithms.}
\label{fig:old-vs-new}
\end{center}
\vskip -0.2in
\end{figure*}

\paragraph{Data augmentations.}
To simulate interventions on $\xc{t}$ and learn invariant representations, we design augmentation procedures which construct augmented data given an input and an algorithm step, such that the step of the algorithm is the same on the original input and on the augmented data. 

We consider simple augmentations, which we describe in detail in \Cref{app:augm}. To reduce the computational overhead, given an input graph, instead of sampling an augmentation at each algorithm step, we sample a single step, $\Tilde{t} \sim \mathcal{U}\{1,T\} $, and construct an augmentation only for the sampled step. Then, we use the (same) constructed augmentation in all the steps \emph{until} the sampled one, $t \leq \Tilde{t}$. This follows from the consideration that, if augmentations are carefully constructed, the execution of the algorithm is the same not only in the next step but in all steps leading up to that.

Whenever possible, we relax the requirement of having the augmentation with \emph{exactly} the same execution, and we allow for approximate augmentations, in order to avoid over-engineering the methodology and obtain a more robust model. This results in more general and simpler augmentations, though we expect more tailored ones to perform better. We refer the reader to \Cref{app:augm} for more details.

We end this paragraph by stressing that we \emph{never} run the target algorithm on the augmented inputs: rather, we directly construct them to have the same next execution step as the corresponding inputs. As a result, our method does not require direct access to the algorithm used to generate the inputs. Furthermore, the number of nodes in our augmentations is at most one more than the number of nodes in the largest training input example. This means that, in all of our experiments, we still never significantly cross the intended test size distribution shift during training.

\paragraph{Datasets.}
We run our method on a diverse subset of the algorithms present in the CLRS benchmark consisting of:
\begin{enumerate*}
    \item \emph{DFS-based algorithms} (Articulation Points, Bridges, Strongly Connected Components \citep{aho1974design}, Topological Sort \citep{knuth1973fundamental});
    \item \emph{Other graph-based algorithms} (Bellman-Ford \citep{bellman1958routing}, BFS \citep{moore1959shortest}, DAG Shortest Paths, Dijkstra \citep{dijkstra1959note}, Floyd-Warshall \citep{floyd1962algorithm}, MST-Kruskal \citep{kruskal1956shortest}, MST-Prim \citep{prim1957shortest});
    \item \emph{Sorting algorithms} (Bubble Sort, Heapsort \citep{williams1964algorithm}, Insertion Sort, Quicksort \citep{hoare1962quicksort});
    \item \emph{Searching algorithms} (Binary-search, Minimum).
\end{enumerate*}
This subset is chosen as it contains most algorithms suffering from out-of-distribution performance drops in current state-of-the-art; see \citet[Table 2]{ibarz2022a}.

\paragraph{Results.} \Cref{fig:old-vs-new} compares the \gls{ood} performances of the Triplet-GMPNN baseline, which we have re-trained and evaluated in our experiments, to our model \ourmethod, as described above. \ourmethod{} performs better or comparable to the existing state-of-the-art baseline,
showcasing how the proposed procedure appears to be beneficial not only theoretically, but also in practice.
The most significant improvements can be found in the sorting algorithms, where we obtain up to $3\times$ increased performance.

\subsection{Ablation study}
\label{exp:ablation}
In this section we study the contribution and importance of two main components of our methodology.
First, we consider the impact of the change we made with respect to the original baseline proposed in \citet{ibarz2022a}, namely the inclusion of the reversal of hint pointers. Second, as we propose a novel way to leverage hints through our self-supervised objective, which is different from the direct supervision in the baseline, one may wonder whether completely removing hints can achieve even better scores. Thus, we also study the performance when completely disregarding hints and directly going from input to output. Finally, we refer the reader to \Cref{app:exp} for additional ablation experiments, including the removal of the KL component in \Cref{eq:contrastive}---which is necessary for the theoretical results but may not always be needed in practice.

\paragraph{The effect of the inclusion of pointers' reversal.}
\input{table-inversion}
As discussed above, pointers' reversal simply consists of adding an additional hint for each hint of \texttt{pointer} type (if any), such that a node not only has the information representing which other node it points to, but also from which nodes it is pointed by.
We study the impact of this inclusion by running the baseline with these additional hints, and evaluate its performance against both the baseline and our \ourmethod.
\Cref{tab:test_results_inv_point} shows that this addition, which we refer to as \textbf{Baseline + reversal}, indeed leads to improved results for certain algorithms, but does not obtain the predictive performances we reached with our regularisation objective.

\paragraph{The removal of hints.}
\input{table-nohints}
While previous works directly included the supervision on the hint predictions, we argue in favour of a novel way of leveraging hints. We use hints first to construct the augmentations
representing the same algorithm step, and then we employ their representations in the self-supervised objective. An additional valid model might consist of a model that directly goes from input to output and completely ignores hints. In \Cref{tab:test_results_nohints} we show that this \textbf{No Hints} model can achieve very good performances, but it is still generally outperformed by \ourmethod.

\section{Conclusions}

In this work we propose a self-supervised learning objective that employs augmentations derived from available hints, which represent intermediate steps of an algorithm, as a way to better ground the execution of GNN-based algorithmic reasoners on the computation that the target algorithm performs. Our \ourmethod{} model, based on such self-supervised objective, leads to algorithmic reasoners that produce more robust outputs of the target algorithms, especially compared to autoregressive hint prediction. 
In conclusion, hints can take you a long way, if used in the right way.

\section*{Acknowledgements}
The authors would like to thank Andrew Dudzik and Daan Wierstra for valuable feedback on the paper. They would also like to show their gratitude to the Learning at Scale team at DeepMind for a supportive atmosphere. 

\bibliography{example_paper}
\bibliographystyle{icml2023}

\newpage
\appendix
\onecolumn

\input{app}

\end{document}

%% file: table-inversion.tex
\begin{table}[t]
\centering
\caption{\emph{Effect of the inclusion of pointers' reversal on each algorithm.} The table shows mean and stderr of the \gls{ood} micro-F$_{1}$ score after $10{,}000$ training steps, across different seeds.}
\vskip 0.1in
\tiny
\begin{tabular}{lccc}
\toprule
\textbf{Alg.} & \textbf{Baseline} & \textbf{Baseline + reversal } & \textbf{\ourmethod{} (ours)} \\
\midrule
Articulation points &  $ 88.93\% \pm 1.92$ & $ 91.04\% \pm  0.92 $ &  ${\bf 98.45}\% \pm  0.60$ \\
Bridges & $ 93.75\% \pm  2.73 $ & $ 97.70\% \pm  0.34 $ & $ {\bf 99.32}\% \pm  0.09$ \\
SCC &  $ 38.53\% \pm  0.45 $ & $ 31.40\% \pm  8.80  $ & $ {\bf 76.79}\% \pm  3.04 $\\
Topological sort & $ 87.27\% \pm 2.67 $ & $ 88.83\% \pm  7.29  $ & $ {\bf 96.59}\% \pm  0.20 $\\
\midrule
Bellman-Ford & $ {\bf 96.67}\% \pm  0.81 $ & $ 95.02\% \pm  0.49 $ & $ 95.54\% \pm  1.06 $ \\
BFS & $ 99.64\% \pm  0.05 $ & $ {\bf 99.93}\% \pm  0.03 $ & $ 99.00\% \pm  0.21 $ \\
DAG Shortest Paths & $ 88.12\% \pm  5.70 $ & $ 96.61\% \pm  0.61 $ & $  {\bf 98.17}\% \pm  0.26 $ \\
Dijkstra & $ 93.41\% \pm  1.08 $ & $ 91.50\% \pm  1.85 $ & $ {\bf 97.74}\% \pm  0.50 $ \\
Floyd-Warshall & $ 46.51\% \pm  1.30  $ & $ 46.28\% \pm  0.80 $ & $ {\bf 72.23}\% \pm  4.84 $\\
MST-Kruskal & $ 91.18\% \pm  1.05 $ & $ 89.93\% \pm  0.43 $ & $ {\bf 96.01}\% \pm  0.45 $ \\
MST-Prim & $ 87.64\% \pm  1.79 $ & $ 86.95\% \pm  2.34$ & $ {\bf 87.97}\% \pm  2.94 $ \\
\midrule
Insertion sort & $ 75.28\% \pm  5.62 $ & $ 87.21\% \pm  2.80 $ & $ {\bf 92.70}\% \pm  1.29 $\\
Bubble sort & $ 79.87\% \pm  6.85 $ & $ 80.51\% \pm  9.10 $ & $ {\bf 92.94}\% \pm  1.23  $\\
Quicksort & $ 70.53\% \pm  11.59 $ & $ 85.69\% \pm  4.53 $ & $ {\bf 93.30}\% \pm  1.96 $\\
Heapsort & $ 32.12\% \pm  5.20 $ & $ 49.13\% \pm  10.35 $ & $ {\bf 95.16}\% \pm  1.27 $\\
\midrule
Binary Search & $ 74.60\% \pm  3.61 $ & $ 50.42\% \pm  8.45 $ & $ {\bf 89.68}\% \pm  2.13 $ \\
Minimum  & $ 97.78\% \pm  0.63 $ & $ 98.43\% \pm  0.01 $ & $  {\bf 99.37}\% \pm  0.20 $ \\
\bottomrule
\end{tabular}
\label{tab:test_results_inv_point}
\end{table}

%% file: table-nohints.tex
\begin{table}[t]
\centering
\caption{\emph{Importance of hint usage in the final performance.} The table shows mean and stderr of the \gls{ood} micro-F$_{1}$ score after $10{,}000$ training steps, across different seeds.
}
\vskip 0.1in
\tiny
\begin{tabular}{lcc}
\toprule
\textbf{Alg.} & \textbf{No Hints} & \textbf{\ourmethod{} (ours)} \\
\midrule
Articulation points & $ 81.97\% \pm  5.08 $ & $ {\bf 98.45}\% \pm  0.60$ \\
Bridges & $ 95.62\% \pm  1.03 $ & $ {\bf 99.32}\% \pm  0.09$ \\
SCC & $57.63\% \pm  0.68$ & $ {\bf 76.79}\% \pm  3.04 $\\
Topological sort & $84.29\% \pm  1.16$ & $ {\bf 96.59}\% \pm  0.20 $\\
\midrule
Bellman-Ford & $ 93.26\% \pm  0.04 $ & $ {\bf 95.54}\% \pm  1.06 $ \\
BFS & $ {\bf 99.89}\% \pm  0.03 $ & $ 99.00\% \pm  0.21 $ \\
DAG Shortest Paths & $ 97.62\% \pm  0.62 $ & $ {\bf 98.17}\% \pm  0.26 $ \\
Dijkstra & $ 95.01\% \pm  1.14 $ & $ {\bf 97.74}\% \pm  0.50 $ \\
Floyd-Warshall & $ 40.80\% \pm  2.90 $ & $ {\bf 72.23}\% \pm  4.84 $\\
MST-Kruskal & $ 92.28\% \pm  0.82 $ & $ {\bf 96.01}\% \pm  0.45 $ \\
MST-Prim & $ 85.33\% \pm  1.21 $ & $ {\bf 87.97}\% \pm  2.94 $ \\
\midrule
Insertion sort & $ 77.29\% \pm  7.42 $ & $ {\bf 92.70}\% \pm  1.29 $\\
Bubble sort & $ 81.32\% \pm  6.50 $ & $ {\bf 92.94}\% \pm  1.23  $\\
Quicksort & $ 71.60\% \pm  2.22 $ & $ {\bf 93.30}\% \pm  1.96 $\\
Heapsort & $ 68.50\% \pm  2.81 $ & $ {\bf 95.16}\% \pm  1.27 $\\
\midrule
Binary Search & $ {\bf 93.21}\% \pm  1.10 $ & $ 89.68\% \pm  2.13 $ \\
Minimum & $ 99.24\% \pm  0.21 $ & $  {\bf 99.37}\% \pm  0.20 $ \\
\bottomrule
\end{tabular}
\label{tab:test_results_nohints}
\end{table}

%% file: app.tex
\section{Derivation of our Self-Supervised Objective}
\label{app:contrastive}
\Cref{eq:contrastive} represents our objective function to optimise, which we derived by adapting the ReLIC objective \citep{mitrovic2021representation} to our causal graph. Note that we propose a unique and novel way of employing a contrastive-learning based objective, which is also different from \citet{mitrovic2021representation}, as we consider positive and negative examples for each hint representation in an input graph to be hint representations in valid augmentations of the graph.
To better understand the equation, in this section we expand the derivation of our objective.

Recall that our goal is to learn $f(X_t, I_t)$ to be an invariant predictor of $\yref{t+1}$ under changes (interventions) of $\xc{t}$. As we do not have access to $\xc{t}$, because we do not know which subset of the input forms $\xc{t}$, we simulate interventions on $\xc{t}$ through data augmentations.
Therefore, our goal becomes to learn $f(X_t, I_t)$ to be an invariant predictor of $\yref{t+1}$ under (valid) augmentations, that is:
\begin{align*}
    p^{\text{do}(a_i)}(\yref{t+1} \vert f(X_t, I_t)) = p^{\text{do}(a_j)}(\yref{t+1} \vert f(X_t, I_t)), \quad \forall a_i, a_j \in \mathcal{A}_{x_t}
\end{align*}
where $\mathcal{A}_{x_t}$ contains all possible valid augmentations at $t$ for $x_t$, and $p^{\text{do}(a)}$ represents the simulation of the intervention on $\xc{t}$ through data augmentation $a$.
Following \citet{mitrovic2021representation}, we enforce this invariance through a regularisation objective, which for every time step $t$ has the following form:
\begin{align*}
    \expec_{X_t} \expec_{\substack{a_{lk}, a_{qm} \\ \sim \mathcal{A}_{x_t} \times \mathcal{A}_{x_t}}} \sum_{i_t} \sum_{b \in \{a_{lk}, a_{qm}\}} \hat{\mathcal{L}}_b(f(X_t, i_t), \yref{t+1}=i_t) \quad \text{s.t.} \quad \text{KL}(p^{\text{do}(a_{lk})}(\yref{t+1} \vert f(X_t, i_t)), p^{\text{do}(a_{qm})}(\yref{t+1} \vert f(X_t, i_t))) \leq \rho,
\end{align*}
for some small number $\rho$.
Since we consider $\hat{\mathcal{L}}$ to be a contrastive learning objective, we take pairs of hint representations, indexed by $i_t$ and $j_t$, to compute similarity scores and use pairs of augmentations $a_{lk} = (a_l, a_k) \in \mathcal{A}_{x_t} \times \mathcal{A}_{x_t}$, that is
\begin{align*}
    p^{\text{do}(a_{lk})}(\yref{t+1} = j_t \vert f(x_t, i_t)) \propto \exp{(\phi(f(x^{a_l}_t, i_t), f(x^{a_k}_t, j_t)))},
\end{align*}
where $f$ is a neural network and $\phi$ is a (learnable) function to compute the similarity between two representations.
Note that, in words, we are computing the similarity of two hint representations (indexed by $i_t$ and $j_t$, respectively) in two data augmentations (obtained from $a_l$ and $a_k$). Now, note that we want the representations of the same hint to be similar in the two augmentations. This means that the hint representation indexed by $i_t$ in $x^{a_l}_t$ must be similar to the hint representation indexed by $i_t$ in $x^{a_k}_t$. Obviously, the same must be true when considering $j_t$ instead of $i_t$. Furthermore, we want representations of different hints to be dissimilar in the two augmentations. 
Putting all together, our objective function at time $t$ can be rewritten as
\begin{align*}
  \mathcal{L}_t =
     -  \sum_{x_t \in \mathcal{D}_t} \!\! \Big( \sum_{a_{lk}} \sum_{i_t} \log \frac{\exp{(\phi(f(x^{a_l}_t, i_t), f(x^{a_k}_t, i_t)))}}{\sum_{j_t\neq i_t}\exp{(\phi(f(x^{a_l}_t, i_t), f(x^{a_k}_t, j_t)))}} 
     - \alpha \sum_{a_{lk}, a_{qm}} \text{KL}(p^{\text{do}(a_{lk})}, p^{\text{do}(a_{qm})})\Big),  
\end{align*}
where $\mathcal{D}_t$ is the dataset containing the snapshots at time $t$ for all the inputs, $i_t, j_t \in I_t$ are two indices, $a_{lk} = (a_l, a_k) \in \mathcal{A}_{x_t} \times \mathcal{A}_{x_t}$ is a pair of augmentations, with $\mathcal{A}_{x_t}$ the set of all possible valid augmentations at $t$ for $x_t$ (which simulate the interventions on $\xc{t}$). Finally, $\alpha$ is the weighting of the KL divergence penalty and $p^{\text{do}(a_{lk})}$ is a shorthand for $p^{\text{do}(a_{lk})}(\yref{t+1} \vert f(x_t, i_t))$.

We note here that $j_t$ is such that the hint representations of $i_t$ and $j_t$ are actually different. In the SCC example in the main text (visualised in \Cref{fig:scc}), the sum over $j_t$ is a sum over all other possible parents of the node $v$ under consideration. \revision{In a DFS's execution, when the hint under consideration is the colour of a node in an input graph, its representation is regularised to be similar to the hint representation of the same node in the augmentation, and dissimilar to all other hint representations in the augmentation corresponding to \emph{different} colours.}

\section{Causal Graph and Representation Learning Components}
In this section we expand on the definition of our causal graph and on the representation learning components, justifying design choices that were left implicit in the main paper due to space constraints. Specifically, we first explain thoroughly the causal relationships among the random variables and then stress how we use those variables in a learning setting.

Recall that our causal graph (\Cref{fig:data-gen}) describes the data generation process of an algorithmic trajectory. We denote by $X_1$ the random variable representing the input to our algorithm, and refer by $X_{t}$ the snapshot at time step $t$ of the algorithm execution on such input, for every time step in the trajectory $t \in [1 \dots T]$.

We assume $X_{t}$ to be generated by two random variables, $\xc{t}$ and $\xs{t}$, which we assume to be \emph{distinct} parts (or splits) of $X_t$, which \emph{together} form the whole $X_t$. We consider $\xc{t}$ to be the part of the snapshot that does not influence the current execution of the algorithm at time step $t$, and can therefore be arbitrarily different without affecting it. We instead denote by $\xs{t}$ the part of the snapshot that determines the current execution of the algorithm at time step $t$, and therefore should not be changed if we do not want to alter the current step execution.

The current execution is represented as hint values on all nodes and/or edges.
We call $\y{t+1}$ the execution (or hint) \emph{on a specific node or edge} chosen accordingly to an index $I_t$. 
Note that the current step hint $\y{t+1}$ is represented with an increment of the time step, following a convention we adopt to indicate that we first need an execution and only then (in the next time step) we materialise its results. Further, note that $\y{t+1}$ represents the algorithmic step in any node or edge, indicating whether or not it is involved in the current execution, thus either encoding that there is no change (and thus it is not involved in the current step) or representing what is its new hint value. 

By definition of $\xc{t}$ and $\xs{t}$, the current step of the algorithm on the node or edge indexed by $I_t$, namely $\y{t+1}$, is determined by $\xs{t}$ only. 

Assuming a Markovian execution, executing one algorithm step gives us $\xc{t+1}$ and $\xs{t+1}$, which form the snapshot we observe $X_{t+1}$. Note that  $\xc{t+1}$ and $\xs{t+1}$ are potentially different from $\xc{t}$ and $\xs{t}$, because the current execution might now be determined by very different subsets. Finally, note that we do not need an arrow from $\y{t+1}$ to $\xc{t}$ or $\xs{t}$, because $\y{t+1}$ is \emph{deterministically} determined by $\xs{t}$, and therefore all its information can be recovered from $\xs{t}$.

Recall now that our goal is to learn an invariant predictor for the refinement task across changes of $\xc{t}$, as this represents a sufficient condition for the invariance in the prediction of $\y{t+1}$ (see \Cref{theo:inv}). We denote by $\yref{t+1}$ the refinement task of the execution on a specific node or edge chosen according to $I_t$.
We omit the arrow from $I_t$ to $\yref{t+1}$ as the dependency is already implicit through $Y_{t+1}$.

We denote by $f(X_t, I_t)$ the representation of the execution step on a particular node or edge indexed by $I_t$, which we learn to predict $\yref{t+1}$ across changes of $\xc{t}$. Finally, note that $f(X_t, I_t)$ is used by the network to determine the predicted next snapshot, which is determined by the next step prediction, and therefore it has a (dashed) arrow to $X_{t+1}$.

\section{Assumptions on Prior Knowledge of the Unobserved $\xc{t}$}
\revision{
Given a time step $t$, to ensure invariant predictions of $\y{t+1}$ we learn predictors of $\yref{t+1}$ that are invariant across interventions on $\xc{t}$, simulated through data augmentations. However, since $\xc{t}$ is an unobserved random variable, we must make assumptions about its properties to create valid data augmentations. In this section, we clarify our assumptions about prior knowledge of $\xc{t}$ and propose potential methods to eliminate this assumption in future work. 

We start by remarking that our neural network is not assumed to have any prior knowledge about $\xc{t}$. This knowledge is enforced into the network through our regularisation objective (\cref{eq:contrastive}), which is driven by appropriately chosen data augmentations. Indeed, those data augmentations do rely on priors that assume something about what $\xc{t}$ might look like. This is similar to how the choice of data augmentation in image CNNs governs which parts of the image we consider to be ``content'' and ``style''.

However, for most algorithms of interest, the required priors are conceptually very simple, and a single augmentation may be reused for many algorithms. As an example, for many graph algorithms, it is an entirely safe operation to add disconnected subgraphs -- an augmentation we repeatedly employ. Similarly in several sorting tasks, adding elements to the tail end of the list represent a valid augmentation. We provide an exhaustive list of the augmentation for each algorithm in \Cref{app:augm}.

Performing data augmentations without knowledge of $\xc{t}$ represents an interesting but challenging direction, that can be explored in future work. A simple, computationally-expensive, data-augmentation procedure that does not require any knowledge of $\xc{t}$ could consist in randomly augmenting the input graph, run the actual algorithm and consider the generated graph as a valid augmentation \emph{only if} the next step execution of the algorithm remains unaltered. A more interesting approach would consist in learning valid augmentations of a given input, perhaps by meta-learning conserved quantities in the spirit of Noether Networks \citep{alet2021noether}. Investigating these cases remains an important avenue for future research.
}

\section{Theoretical Analysis}
\label{app:theory}

\invariancetheo*
\begin{proof}[Proof of \Cref{theo:inv}.]
\begin{align*}
    &p^{\text{do}(\xc{t})=x}(\y{t+1} \vert f(X_t, I_t)) \\
    &= \int p^{\text{do}(\xc{t})=x}(\y{t+1} \vert \yref{t+1}) p^{\text{do}(\xc{t})=x}(\yref{t+1} \vert f(X_t, I_t)) d \yref{t+1} \\
    &= \int p(\y{t+1} \vert \yref{t+1}) p^{\text{do}(\xc{t})=x}(\yref{t+1} \vert f(X_t, I_t)) d \yref{t+1}
    \\
    &= \int p(\y{t+1} \vert \yref{t+1}) p^{\text{do}(\xc{t})=x'}(\yref{t+1} \vert f(X_t, I_t)) d \yref{t+1}
    \\
    &= p^{\text{do}(\xc{t})=x'}(\y{t+1} \vert f(X_t, I_t)).
\end{align*}
The first equality is obtained by marginalising over $\yref{t+1}$ and using the assumption of $\yref{t+1}$ being a refinement of $\y{t+1}$, which implies that $\yref{t+1}$ has all the necessary information to predict $\y{t+1}$ (and thus we can drop the conditioning on $f(X_t, I_t)$). The second equality follows from the fact that the mechanism $\y{t+1} \vert \yref{t+1}$ is independent of interventions on $\xc{t}$ under our assumptions. Finally, the third equality follows from the assumption that $f(X_t, I_t)$ is an invariant predictor of $\yref{t+1}$ under changes in $\xc{t}$.
\end{proof}

\section{Data Augmentations}
\label{app:augm}
In this section we expand on our proposed augmentations, which simulate interventions on $\xc{t}$, valid \emph{until} step $t$. Further, we report which hints we use in our objective (see \Cref{eq:contrastive}) using the naming convention in \citet{velivckovic2022clrs}.

\paragraph{DFS-based algorithms (Articulation Points, Bridges, Strongly Connected Components, Topological Sort).} We construct exact augmentations for these kinds of problems. First, we sample a step by choosing uniformly at random amongst those where we enter a node for the first time (in case multiple DFSs are being executed for an input, we only consider the first one). Then, we construct a subgraph of nodes with larger node-ids, and we randomly determine connectivity between subgraph's nodes. Finally, we connect all the subgraph's nodes to the node we are entering in the sampled step.
We contrast the following hints up to the sampled step (we mask out the contrastive loss on later steps): \texttt{pi\_h} in Articulation Points and Bridges; \texttt{scc\_id\_h}, \texttt{color} and \texttt{s\_prev} in Strongly Connected Components; and \texttt{topo\_h}, \texttt{color}, and \texttt{s\_prev} in Topological Sort.

\paragraph{Graph-based algorithms (Bellman-Ford, BFS, DAG Shortest Path, Dijkstra, Floyd-Warshall, MST-Kruskal, MST-Prim).} We construct simple but exact augmentations consisting of adding a disconnected subgraph to each input's graph. The subgraph consists of nodes with larger nodes ids and whose connectivity is randomly generated. We contrast until the end of the input's trajectory the following hints: \texttt{pi\_h} in Bellman-Ford, BFS, Dijkstra and MST-Prim; \texttt{pi\_h}, \texttt{topo\_h}, \texttt{color} in DAG shortest path; \texttt{Pi\_h} in Floyd-Warshall; \texttt{pi} in MST-Kruskal.

\paragraph{Sorting algorithms (Insertion sort, Bubble Sort, Quicksort, Heapsort).} We construct general augmented inputs obtained by simply adding items at the end of each input array. We consider as trajectories for those augmentations the ones of the corresponding inputs. We note that those do not correspond to exact augmentations for all sorting algorithms, but only for Insertion Sort. Indeed, running the executor of one of the other algorithms would yield potentially different trajectories than those we consider. However, since we use the inputs' trajectories, our regularisation aims at learning to be invariant to added nodes that do not contribute to each currently considered step.
We contrast until the end of the input's trajectories the hints \texttt{pred\_h}, and \texttt{parent} for Heapsort, and \texttt{pred\_h} for all the other algorithms.

\paragraph{Searching algorithms (Binary Search, Minimum).} We construct general augmented inputs obtained by simply adding random numbers (different than the searched one) at the end of the input array. For those augmentations, we consider as trajectories the ones of the corresponding input arrays, whose hints are contrasted until the end of the trajectories themselves. We remark that running the searching algorithms on such augmentations could potentially lead to ground-truth trajectories different than those of the inputs. However, since we consider as trajectories for the augmentations the inputs' ones, the contrastive objective is still valid, and can be seen as pushing the hint representations to be invariant to messages coming from nodes that are not involved in the current computation. We run our model by allowing the network to predict the predecessor of every array's item, namely \texttt{pred\_h}, at every time step and use its representation in our regularisation loss (in other words, we do not run with the static hint elimination of \citet{ibarz2022a}).

\section{Experiments}

\begin{figure*}[t]
\vskip 0.2in
\begin{center}
\centerline{\includegraphics[width=\textwidth]{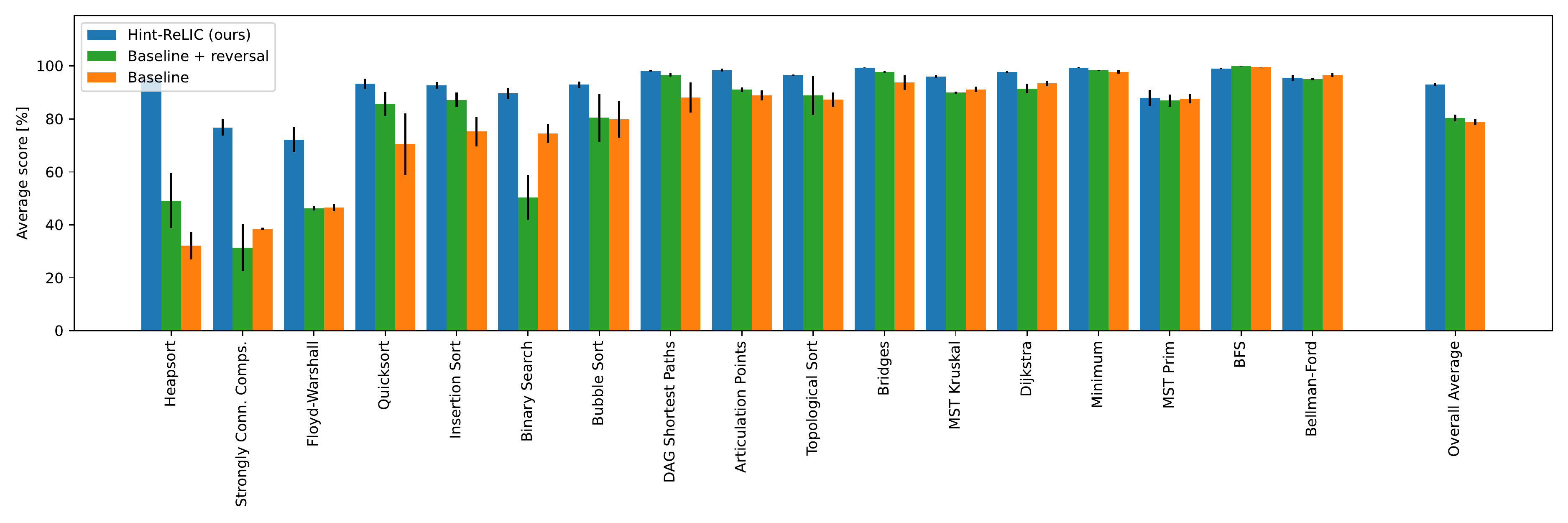}}
\caption{Per-algorithm comparison of the Triplet-GMPNN baseline \citep{ibarz2022a}, its augmented version which includes pointers' reversal and our \ourmethod{}. Error bars represent
the standard error of the mean across three random seeds. The final column shows the average and standard error of the mean performances across the different algorithms.}
\label{fig:old-vs-new-vs-oldinv}
\end{center}
\vskip -0.2in
\end{figure*}

\subsection{Additional experiments}
\label{app:exp}
\Cref{tab:test_results_grouping} contains a comprehensive set of experiments, including the performances of the \textbf{No Hints}, \textbf{Baseline}, \textbf{Baseline + reversal} models, as discussed in the main text. The column \textbf{Baseline + reversal + contr. + KL} represents our \ourmethod{} model, which is obtained with the additional inclusion of the contrastive and KL losses (see \Cref{eq:contrastive}).
Additionally, we report performances of our model when \emph{removing} the KL divergence loss (setting $\alpha=0$ in \Cref{eq:contrastive}), namely \textbf{Baseline + reversal + contr.}. By comparing \textbf{Baseline + reversal + contr.} to \textbf{Baseline + reversal + contr. + KL}, we can see that, even if the KL penalty produces some gain for certain algorithms, it does not represent the component leading to the most improvement.
Finally, \Cref{tab:test_results_grouping} also reports the scores obtained in the DFS algorithm, which appears to be solved by the inclusion of the pointers' reversal. We do not run our contrastive objective on such algorithm as there is no additional improvement to be made.

Finally, to further evaluate the impact of the pointers' reversal, we report the performances of \ourmethod{} without the inclusion of such additional hints. As can be seen in \Cref{tab:test_results_oursnoinv}, the pointers' reversal helps stabilise our model, especially in the sorting algorithms. We remark however how only including those pointers' reversal into a baseline model does not produce the performances of our model (see column \textbf{Baseline + reversal} in \Cref{tab:test_results_grouping} and \Cref{fig:old-vs-new-vs-oldinv}).

\begin{table*}[t]
\centering
\caption{\emph{Comparison of performances for different models, with last column representing our proposed method \ourmethod.} Table shows mean and stderr of \gls{ood} micro-F$_{1}$ score after $10{,}000$ training steps, across different seeds.
}
\vskip 0.1in
\small
\begin{tabular}{lccccc}
\toprule
\textbf{Alg.} & \textbf{No Hints} & \textbf{Baseline} & \textbf{Baseline} & \textbf{Baseline} & \textbf{Baseline} \\
& & & \textbf{+ reversal}&  \textbf{+ reversal + contr.}& \textbf{+ reversal + contr. + KL} \\
\midrule
Articulation points & $ 81.97\% \pm  5.08 $ & $ 88.93\% \pm 1.92$ & $ 91.04\% \pm  0.92 $ & $ {\bf 98.91}\% \pm  0.34 $ & $ 98.45\% \pm  0.60$ \\
Bridges & $ 95.62\% \pm  1.03 $ & $ 93.75\% \pm  2.73 $ & $ 97.70\% \pm  0.34 $ & $ 98.14\% \pm  2.00 $ & $ {\bf 99.32}\% \pm  0.09$ \\
DFS & $ 33.94\% \pm  2.57 $ & $ 39.71\% \pm  1.34 $ & $ {\bf 100.00}\% \pm  0.00 $ & $ - $ & $ - $ \\
SCC & $57.63\% \pm  0.68$ & $ 38.53\% \pm  0.45 $ & $ 31.40\% \pm  8.80  $ & $ 75.78\% \pm  1.25 $ & $ {\bf 76.79}\% \pm  3.04 $\\
Topological sort & $84.29\% \pm  1.16$ & $ 87.27\% \pm 2.67 $ & $ 88.83\% \pm  7.29  $ & $ 95.44\% \pm  0.52 $ & $ {\bf 96.59}\% \pm  0.20 $\\
\midrule
Bellman-Ford & $ 93.26\% \pm  0.04 $ & $ {\bf 96.67}\% \pm  0.81 $ & $ 95.02\% \pm  0.49 $ & $ 95.26\% \pm  0.92 $ & $ 95.54\% \pm  1.06 $ \\
BFS & $ 99.89\% \pm  0.03 $ & $ 99.64\% \pm  0.05 $ & $ {\bf 99.93}\% \pm  0.03 $ & $ 98.41\% \pm  0.39 $ & $ 99.00\% \pm  0.21 $ \\
DAG Shortest Paths & $ 97.62\% \pm  0.62 $ & $ 88.12\% \pm  5.70 $ & $ 96.61\% \pm  0.61 $ & $ 97.31\% \pm  0.51 $ & $ {\bf 98.17}\% \pm  0.26 $ \\
Dijkstra & $ 95.01\% \pm  1.14 $ & $ 93.41\% \pm  1.08 $ & $ 91.50\% \pm  1.85 $ & $ 97.22\% \pm  0.12 $ & $ {\bf 97.74}\% \pm  0.50 $ \\
Floyd-Warshall & $ 40.80\% \pm  2.90 $ & $ 46.51\% \pm  1.30  $ & $ 46.28\% \pm  0.80 $ & $ 71.43\% \pm  2.64 $ & $ {\bf 72.23}\% \pm  4.84 $\\
MST-Kruskal & $ 92.28\% \pm  0.82 $ & $ 91.18\% \pm  1.05 $ & $ 89.93\% \pm  0.43 $ & $ 95.18\% \pm  1.29 $ & $ {\bf 96.01}\% \pm  0.45 $ \\
MST-Prim & $ 85.33\% \pm  1.21 $ & $ 87.64\% \pm  1.79 $ & $ 86.95\% \pm  2.34$ & $ {\bf 89.23}\% \pm  1.23 $ & $ 87.97\% \pm  2.94 $ \\
\midrule
Insertion sort & $ 77.29\% \pm  7.42 $ & $ 75.28\% \pm  5.62 $ & $ 87.21\% \pm  2.80 $ & $ {\bf 95.06}\% \pm  1.33 $ & $ 92.70\% \pm  1.29 $\\
Bubble sort & $ 81.32\% \pm  6.50 $ & $ 79.87\% \pm  6.85 $ & $ 80.51\% \pm  9.10 $ & $ {\bf 94.09}\% \pm  0.80 $ & $ 92.94\% \pm  1.23  $\\
Quicksort & $ 71.60\% \pm  2.22 $ & $ 70.53\% \pm  11.59 $ & $ 85.69\% \pm  4.53 $ & $ 90.54\% \pm  2.49 $ & $ {\bf 93.30}\% \pm  1.96 $\\
Heapsort & $ 68.50\% \pm  2.81 $ & $ 32.12\% \pm  5.20 $ & $ 49.13\% \pm  10.35 $ & $ 89.41\% \pm  4.79  $ & $ {\bf 95.16}\% \pm  1.27 $\\
\midrule
Binary Search & $ {\bf 93.21}\% \pm  1.10 $ & $ 74.60\% \pm  3.61 $ & $ 50.42\% \pm  8.45 $ & $ 87.50\% \pm  3.62 $ & $ 89.68\% \pm  2.13 $ \\
Minimum & $ 99.24\% \pm  0.21 $ & $ 97.78\% \pm  0.63 $ & $ 98.43\% \pm  0.01 $ & $ 99.54\% \pm  0.05 $ & $  {\bf 99.37}\% \pm  0.20 $ \\
\bottomrule
\end{tabular}
\label{tab:test_results_grouping}
\end{table*}

\input{table-ours-noinv}

\subsection{Implementation details}
We use the best hyperparameters of the Triplet-GMPNN~\citep{ibarz2022a} base model, and we only reduce the batch size to 16. We set the temperature parameter $\tau$ to $1e-1$ and the weight of the KL loss $\alpha$ to 1. We implement the similarity function as $\phi(f(x^{a_l}_t, i_t), f(x^{a_k}_t, i_t)) = \langle\,h(f(x^{a_l}_t), i_t), h(f(x^{a_k}_t, i_t))\,\rangle / \tau$
with $h$ a two-layers MLP with hidden and output dimensions equal to the input one, and ReLU non-linearities.

%% file: table-ours-noinv.tex
\begin{table}[t]
\centering
\caption{\emph{Importance of the inclusion of the pointers' reversal in our \ourmethod{}.} The table shows mean and stderr of the \gls{ood} micro-F$_{1}$ score after $10{,}000$ training steps, across different seeds.
}
\vskip 0.1in
\small
\begin{tabular}{lcc}
\toprule
\textbf{Alg.} & \textbf{\ourmethod{}} & \textbf{\ourmethod{}} \\
& & (no reversal) \\
\midrule
Articulation points & $ {\bf 98.45}\% \pm  0.60$ & $  97.33\% \pm  1.32 $\\
Bridges & $ 99.32\% \pm  0.09$ & $ {\bf 99.42}\% \pm  0.20 $\\
SCC &  $ 76.79\% \pm  3.04 $ & $ {\bf 81.42}\% \pm  2.68 $\\
Topological sort & $ {\bf 96.59}\% \pm  0.20 $ & $ 80.25\% \pm  3.03 $\\
\midrule
Bellman-Ford & $ {\bf 95.54}\% \pm  1.06 $ & $ 95.27\% \pm  0.97 $ \\
BFS & $ {\bf 99.00}\% \pm  0.21 $ & $ 98.23\% \pm  0.17 $ \\
DAG Shortest Paths & $ {\bf 98.17}\% \pm  0.26 $ & $ 89.23\% \pm  7.11 $ \\
Dijkstra & $ {\bf 97.74}\% \pm  0.50 $ & $ 96.70\% \pm  0.92 $ \\
Floyd-Warshall & $ {\bf 72.23}\% \pm  4.84 $ & $ 57.38\% \pm  1.75 $ \\
MST-Kruskal & $ {\bf 96.01}\% \pm  0.45 $ & $ 94.53\% \pm  0.40 $ \\
MST-Prim & $ {\bf 87.97}\% \pm  2.94 $ & $ 74.24\% \pm  10.85 $ \\
\midrule
Insertion sort & $ {\bf 92.70}\% \pm  1.29 $ & $ 67.80\% \pm  10.86 $\\
Bubble sort & $ {\bf 92.94}\% \pm  1.23 $ & $ 82.36\% \pm  6.88 $ \\
Quicksort & $ {\bf 93.30}\% \pm  1.96 $ & $ 74.32\% \pm  10.12 $ \\
Heapsort & $ {\bf 95.16}\% \pm  1.27 $ & $ 77.15\% \pm  4.73 $\\
\midrule
Binary Search & $ {\bf 89.68}\% \pm  2.13 $ & $ 86.65\% \pm  2.38 $ \\
Minimum & $ {\bf 99.37}\% \pm  0.20 $ & $ 98.91\% \pm  0.23 $ \\
\bottomrule
\end{tabular}
\label{tab:test_results_oursnoinv}
\end{table}